\documentclass[letterpaper]{article} 
\usepackage[preprint]{aaai2027}
\usepackage[hyphens]{url}  
\usepackage{graphicx}      
\usepackage{natbib}        
\usepackage{caption}       
\usepackage{amsmath,amssymb,amsfonts}
\usepackage{booktabs}
\usepackage{multirow}
\usepackage{array}
\usepackage{enumitem}
\usepackage{algorithm}
\usepackage{algorithmic}
\usepackage{amsthm}
\newtheorem{proposition}{Proposition}
\title{MGSB: Manifold Gated Signature Branch
Pressure-Domain Baseline Architecture
for Two-Phase Pipeline Flows Under Distributional Shift}
\author{Issah Suleiman, Sormeh Serpoosh, Nadine Elkholy, Hicham Ferroudji, Mohammad Azizur Rahman, Matthew Hamilton}
\affiliations{}

\begin{document}
\maketitle

\begin{abstract}
Leak detection models for multiphase pipelines often degrade when
deployed under flow regimes that differ from training. Existing
evaluations typically assess performance under in-distribution
operating conditions, masking failures caused by regime transitions
such as bubble-to-slug flow. We propose the Manifold Gated Signature
Bias (MGSB), a regime-aware architecture combining regime-conditioned
feature fusion, a TT-RoughPath encoder, and Mean-Teacher consistency
regularization to improve robustness under distribution shift. Under
leave-one-group-out evaluation, MGSB achieves a detection F1 of 0.930
and an OOD F1 of 0.783, substantially outperforming CNN-LSTM and
fully connected baselines under severe feature corruption. Ablations
show the proposed architecture, not the training procedure, is the
primary contributor to OOD robustness, while Mahalanobis-distance
analysis confirms the held-out conditions are genuinely
out-of-distribution. These results show that explicit regime-aware
modelling is a practical path toward robust, sensor-agnostic leak
detection in industrial multiphase pipelines.
\end{abstract}

\section{Introduction}

A 1.8-mm subsea pipeline leak can spread contamination within
minutes~\citep{Datta2021,Meribout2021}. Pressure sensors are the
standard deployed response: cheap, robust, and fast enough to catch
slug-induced transients~\citep{Hunaidi1999,Wan2020}. The challenge is
not measurement but building a classifier that stays useful when the
flow regime at the monitoring site differs from training---yet
existing detectors fail this challenge silently. A model trained on
bubble flow generalizes poorly to slug flow, where large gas pockets
produce pressure surges an order of magnitude larger than the leak
signal~\cite{ferroudji2026,ferroudji2025}. Published detectors often
look strong because train and test data share operating conditions,
masking a structural mismatch that worsens as fields age and regimes
shift~\citep{Wang2021,Gama2014}. Pressure-only sensing remains the
preferred industrial configuration for its low cost and standard
instrumentation, yet despite advances in deep learning for
pressure-based leak detection~\citep{Zhang2019leak,Chen2021pressure,
Vaswani2017,Wu2021}, existing architectures remain largely
regime-agnostic. Mean-Teacher consistency
training~\citep{Tarvainen2017} improves robustness under covariate
shift, but has not been integrated with regime-aware fusion for
multiphase pipeline monitoring.

Our contributions are:
\begin{enumerate}[nosep,leftmargin=*]
\item A \textbf{Regime-Conditioned Manifold Gate} for adaptive
fusion under flow-regime transitions, providing a physics-grounded
fallback under input corruption and distributional shift.

\item A \textbf{TT-RoughPath} representation combining
rough-path geometry and tensor-train compression for
geometry-aware pressure trajectory encoding. At the low sampling
rate (2\,Hz) considered here its contribution to OOD robustness is
modest but is expected to grow at higher sampling rates.

\item An \textbf{Adaptive Input Compressor (AIC)} that projects
inputs of arbitrary dimensionality into a shared 32-dimensional
representation, enabling a pressure-trained model to generalize to
the 1{,}459-feature GPLA dataset without retraining.

\item A regime-aware learning framework that integrates
Mean-Teacher consistency training with adaptive fusion, demonstrating
robust generalization across internal flow-loop experiments and two
external datasets.

\item Extensive leave-one-out (LOO) and cross-dataset experiments,
supported by Mahalanobis-distance analysis in manifold space to
characterize out-of-distribution generalization.
\end{enumerate}

\section{Related Work}

\textbf{Pressure-based detection and regime classification.}
Statistical gradient-threshold methods~\citep{Wan2020} are reliable
under steady flow but fail during regime transitions, and
model-based observers~\citep{Verde2005,Datta2021} are sensitive to
parameter error; ML methods~\citep{Zhang2019leak,Chen2021pressure,
Meribout2021} improve single-phase performance but rarely address
inter-regime shift directly. Regime identification itself is
typically done via Electrical Resistance Tomography
(ERT)~\citep{Jia2015}, gamma densitometry, or visual
inspection~\citep{Taitel1980,Barnea1987}; prior neural
classifiers~\citep{Roshani2017,Hu2020regime} use CNN or LSTM
architectures but do not propagate regime information into downstream
detection layers, limiting cross-condition transferability.

\textbf{Domain adaptation, OOD robustness, and signature methods.}
Mean-Teacher~\citep{Tarvainen2017} uses EMA pseudo-labels to improve
covariate-shift robustness, and FixMatch~\citep{Sohn2020} adds a
confidence threshold; domain-adversarial
training~\citep{Wang2021,Jin2022ood} assumes target-domain data
during training, which is unavailable in LOO environments. Separately,
path signatures~\citep{Chen1954,Lyons1998} represent time series
coordinate-independently, and Tensor-Train
Decomposition~\citep{Oseledets2011} reduces signature memory from
$O(D^d)$ to $O(dDR^2)$; TT-decomposed path signatures have not
previously been applied to pipeline pressure anomaly detection.

\textbf{Regime-conditioned architectures.}
Mixture-of-Experts~\citep{Jacobs1991,Shazeer2017} routes inputs
conditionally but requires expert assignment. FiLM
conditioning~\citep{Perez2018} applies feature-wise affine
transforms uniformly across channels. The proposed gate instead
computes channel-specific attention weights conditioned on the
estimated flow regime, allowing pressure pathways to
contribute asymmetrically across bubble, plug, and slug conditions.
Unlike domain-adversarial methods, our conditioning signal derives
from a physics-informed mapping of liquid and gas superficial
velocities ($V_{SL}$ and $V_{SG}$, respectively) to flow regime,
requiring no target-domain data~\citep{Gama2014}.

\section{Methodology}

MGSB processes a pressure window in five stages: (1) an Adaptive Input
Compressor normalises the raw channels; (2) a DepthwiseSE encoder and
(3) a TT-RoughPath encoder each summarise the window into a
32-dimensional pressure descriptor, which are combined into a single
representation $\mathbf{h}_p$; (4) a Regime Classifier predicts the
active flow regime from $\mathbf{h}_p$; (5) a Regime-Conditioned
Manifold Gate fuses $\mathbf{h}_p$ with the predicted regime to
produce the manifold $\mathbf{m}$ that feeds the output heads. Each
stage is defined in turn below.

\paragraph{Problem formulation.}
We consider four input features computed from the upstream ($P_1$) and
downstream ($P_2$), post-leak, pressure measurements: the pressure difference
($P_{\text{diff}} = P_1 - P_2$), absolute pressure ($P_{\text{abs}} = |P_{diff}|$), ($P_{\mathrm{abs}}(t)=\lvert P_{\mathrm{diff}}(t)\rvert$), the rolling
root-mean-square differential pressure
$P_{\mathrm{roll\_rms}}(t)
=
\sqrt{\frac{1}{T}\sum_{i=t-T+1}^{t}P_{\mathrm{diff}}(i)^2}$,
and the rolling standard deviation
$P_{\mathrm{roll\_std}}(t)
=
\sqrt{\frac{1}{T-1}\sum_{i=t-T+1}^{t}
\left(P_{\mathrm{diff}}(i)-\overline{P}_{\mathrm{diff},t}\right)^2}$,
where
$\overline{P}_{\mathrm{diff},t}
=
\frac{1}{T}\sum_{i=t-T+1}^{t}P_{\mathrm{diff}}(i).$
Each time-series window
$\mathbf{x}\in\mathbb{R}^{T\times C}$ ($T=10$ timesteps, $C=4$
channels) is used to predict leak presence $y_\text{det}\in\{0,1\}$
and flow regime $y_r\in\{0,1,2,3\}$ (bubble/plug/slug/OOD).

\paragraph{Stage 1: Adaptive Input Compressor (AIC).}
A single linear projection maps the $C$ input channels at each
timestep to a fixed 32-dimensional representation,
\begin{equation}
\mathbf{x}_c = \mathrm{LayerNorm}(W_c\mathbf{x}) \in \mathbb{R}^{T\times32},
\end{equation}
where $W_c\in\mathbb{R}^{32\times C}$ is a learned projection matrix
applied independently at each timestep and $\mathrm{LayerNorm}(\cdot)$
normalises each resulting 32-dimensional vector. Because $W_c$'s input
dimension is set by $C$ rather than hard-coded, the same architecture
accepts inputs of any channel count at inference (see
\emph{External validation datasets}, below).

\paragraph{Stage 2: DepthwiseSE Encoder.}
Depthwise convolution $\mathrm{DW}(\cdot)$, pointwise mixing
$W_p\in\mathbb{R}^{32\times32}$, Batch Normalisation
$\mathrm{BN}(\cdot)$, and Squeeze-Excitation channel attention
$\mathrm{SE}(\cdot)$ act on the AIC output $\mathbf{x}_c$ to produce a
residual, channel-reweighted sequence,
\begin{equation}
\mathbf{z} = \mathbf{x}_c +
\mathrm{GELU}(\mathrm{BN}(W_p\mathrm{DW}(\mathbf{x}_c)))
\odot\sigma(\mathrm{SE}(\mathbf{x}_c)) \in \mathbb{R}^{T\times32},
\end{equation}
where $\odot$ is element-wise multiplication, $\sigma(\cdot)$ is the
sigmoid function, and $\mathrm{GELU}(\cdot)$ is the Gaussian Error
Linear Unit activation. Stochastic depth (bypass probability
$10\,\%$) is applied to $\mathbf{z}$ during training to prevent
overfitting. Global average pooling over the temporal axis then
collapses $\mathbf{z}$ into a single encoder vector
$\mathbf{h}_\text{enc} = \tfrac{1}{T}\sum_{t=1}^{T}\mathbf{z}_t \in
\mathbb{R}^{32}$.

\paragraph{Stage 3: TT-RoughPath Encoder.}
In parallel with Stage 2, a geometry-aware encoder summarises the same
window $\mathbf{x}_c$ using rough-path theory. Let
$\Delta\mathbf{x}_t = \mathbf{x}_{c,t}-\mathbf{x}_{c,t-1}$ be the
trajectory increment at step $t$ and
$\hat{\boldsymbol{\pi}} = \mathrm{normalize}\bigl(\sum_{t}\Delta\mathbf{x}_t\bigr)
\in \mathbb{R}^{32}$ the $\ell_2$-normalised cumulative increment
across the window. Three Tensor-Train cores
$\mathcal{G}_k\in\mathbb{R}^{8\times32\times8}$, $k=1,2,3$, of rank
$r=8$ approximate the path signature of $\mathbf{x}_{c,0:T}$
by repeatedly contracting each core with $\hat{\boldsymbol{\pi}}$ along
its 32-dimensional feature mode and chaining the resulting rank-8
vector into the next core~\citep{Lyons1998,Oseledets2011}:
\begin{equation}
\mathbf{s} = \mathcal{G}_1\times_1\hat{\boldsymbol{\pi}}
\times_{(2,3)}\mathcal{G}_2\times_1\hat{\boldsymbol{\pi}}
\times_{(2,3)}\mathcal{G}_3 \in \mathbb{R}^{8},
\end{equation}
where $\times_1$ contracts a core's rank-index with the incoming
rank-8 vector and $\times_{(2,3)}$ jointly contracts the resulting
matrix's feature mode (size 32) with $\hat{\boldsymbol{\pi}}$ and its
remaining rank mode (size 8), yielding the next rank-8 vector; $\mathbf{s}$
is the final rank-8 output after the third core. A learned matrix
$W_s\in\mathbb{R}^{32\times8}$ projects $\mathbf{s}$ back to the shared
32-dimensional space, $\mathbf{h}_\text{sig}=W_s\mathbf{s}\in\mathbb{R}^{32}$.

\paragraph{Encoder fusion.}
The pooled convolutional descriptor and the signature descriptor are
summed to give the final pressure representation used by every
downstream stage,
\begin{equation}
\mathbf{h}_p = \mathbf{h}_\text{enc} + \mathbf{h}_\text{sig} \in \mathbb{R}^{32}.
\end{equation}

\paragraph{Stage 4: Regime Classifier.}
A two-layer MLP predicts the active flow regime from $\mathbf{h}_p$,
\begin{equation}
\hat{r} = \arg\max\,\mathrm{MLP}(\mathbf{h}_p) \in \{0,1,2,3\},
\end{equation}
where the four output classes correspond to bubble, plug, slug, and
OOD. Ground-truth regime labels supervise $\hat{r}$ during training;
at inference the predicted (not ground-truth) regime conditions the
manifold gate below.

\paragraph{Stage 5: Regime-Conditioned Manifold Fusion.}
The fusion gate is the architectural core of MGSB: a scalar weight
$g\in(0,1)$ determines how much the fused representation relies on
the observed pressure features $\mathbf{h}_p$ versus a learned regime
prior. When the input is clean, $g\to1$ and the gate passes the
data-driven representation; when the input is corrupted or
out-of-distribution, $g\to0$ and the gate falls back on the regime
embedding---a bounded, training-stable prior that does not degrade
under input corruption. Let $\mathrm{Emb}(\hat{r})\in\mathbb{R}^{d_\ell}$,
$d_\ell=16$, be a learned embedding table mapping each of the four
regime classes to a 16-dimensional vector, trained jointly with the
gate. The gate query $\mathbf{q}$, scalar gate value $g$, and fused
manifold $\mathbf{m}$ are then computed as
\begin{align}
  \mathbf{q} &= W_q\cdot\mathrm{Emb}(\hat{r})\in\mathbb{R}^{d_\ell},\\
  g &= \sigma\!\left(\tfrac{\mathbf{q}^\top
    W_k\,\mathbf{h}_p}{\sqrt{16}}\right)\in(0,1),\\
  \mathbf{m} &= \mathrm{GELU}\!\left(\mathrm{LN}(g\cdot
      W_v\,\mathbf{h}_p+(1-g)\cdot\mathbf{q})\right)\in\mathbb{R}^{32},
\end{align}
where $W_q\in\mathbb{R}^{d_\ell\times d_\ell}$, $W_k\in\mathbb{R}^{d_\ell\times32}$,
and $W_v\in\mathbb{R}^{32\times32}$ are learned projection matrices,
$\sqrt{16}=\sqrt{d_\ell}$ is the standard dot-product attention scaling
factor~\citep{Vaswani2017}, and $\mathrm{LN}(\cdot)$ is Layer
Normalisation applied for numerical stability before $\mathrm{GELU}$.

\paragraph{Output heads.}
Two linear heads map the manifold $\mathbf{m}$ to the task outputs:
\begin{align}
\hat y_\text{det} &= \sigma(W_d\,\mathbf{m}) \in [0,1],\\
\hat y_\ell &= \sigma(W_\ell\,\mathbf{m}) \in [0,1],
\end{align}
where $W_d,W_\ell\in\mathbb{R}^{1\times32}$ are learned weight vectors,
$\hat y_\text{det}$ is the predicted leak probability, and $\hat
y_\ell$ is a normalised in-window leak-position estimate that is
architecturally present but, as noted below, not meaningfully
supervised in the current dataset.

\begin{figure}[t]
\centering
\includegraphics[width=\columnwidth]{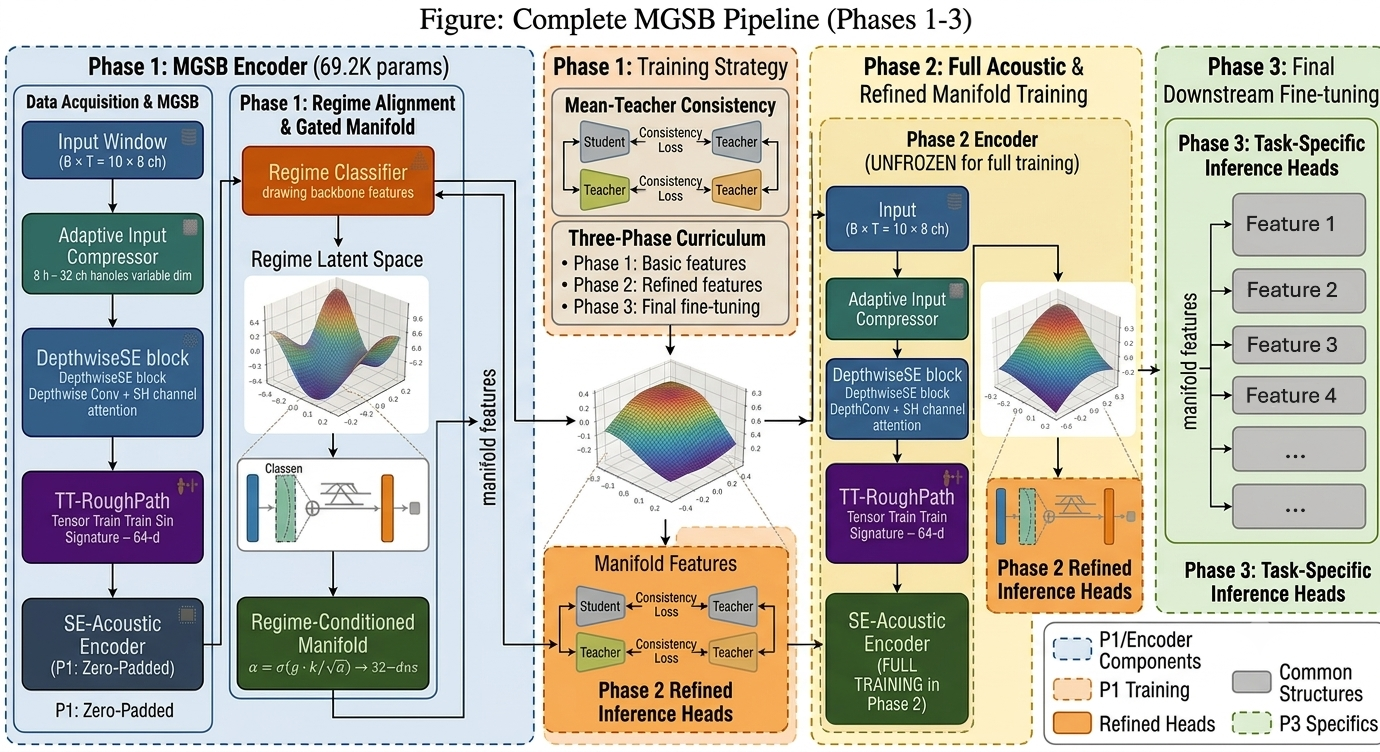}
\caption{MGSB architecture. The AIC compresses each window to a
common 32-dimensional channel space; the DepthwiseSE and TT-RoughPath
encoders each summarise it into a 32-dimensional descriptor, which are
summed into $\mathbf{h}_p$. The Regime Classifier predicts $\hat r$
from $\mathbf{h}_p$; the Regime-Conditioned Manifold Gate fuses
$\mathbf{h}_p$ and $\mathrm{Emb}(\hat r)$ into $\mathbf{m}\in\mathbb{R}^{32}$,
which drives the Detection and Localisation heads.}
\label{fig:arch}
\end{figure}

The gate's OOD stability follows a simple argument: as $g\to0$ under
input corruption, $\mathbf{m}$ falls back entirely on the bounded
regime embedding $\mathbf{q}$, which is invariant to input corruption
by construction. The empirical evidence is the $+0.715$ OOD~F1 gap
over CNN-LSTM+CR ($p<0.001$, $d=5.36$), robust across all 29 folds
(see Results, below).

\paragraph{Mean-Teacher Consistency Regularisation.}
Mean-Teacher regularisation~\citep{Tarvainen2017} maintains two
copies of the network: a \emph{student} $f_\theta$ with parameters
$\theta$ updated by gradient descent, and a \emph{teacher}
$f_{\bar\theta}$ with $\bar\theta$ updated as an exponential moving
average (EMA) of the student, where $f_\theta:\mathbb{R}^{T\times C}
\to[0,1]$ is the full MGSB network mapping an input window to a
detection score. The teacher produces stable pseudo-targets on weakly
augmented inputs $\tilde{\mathbf{x}}_w$, and the student is penalised
for inconsistent predictions on strongly augmented inputs
$\tilde{\mathbf{x}}_s$ of the same window:
\begin{align}
  \mathcal{L}_\text{CR} &= \mathrm{MSE}\bigl(f_\theta(\tilde{\mathbf{x}}_s),
  [f_{\bar\theta}(\tilde{\mathbf{x}}_w)]_\mathrm{SG}\bigr),\\
  \bar\theta &\leftarrow 0.995\,\bar\theta+0.005\,\theta,
\end{align}
where $[\cdot]_\mathrm{SG}$ is a stop-gradient operation. Ablation
confirms $\Delta\text{OOD}=-0.042$ when CR is removed---a secondary
but independently beneficial contribution to OOD robustness.

\paragraph{Training curriculum.}
Training proceeds in three stages. (1)~\textbf{Regime warmup}
(epochs 1--12): only the Regime Classifier and its embedding layer
are updated; all other parameters are frozen, so the manifold gate
receives a stable regime signal before detection training begins.
(2)~\textbf{Detection unlock} (epochs 13--24): the detection head is
unfrozen and trained jointly with the regime modules (the
localisation head is architecturally present but contributes no
meaningful gradient signal, as orifice position labels are
unavailable in the current dataset). (3)~\textbf{Full joint training}
(epochs 25--40): all parameters are updated jointly under cosine
annealing warm restarts~\citep{Bengio2009curriculum} with a reduced
learning rate.

\paragraph{Training objective.}
\begin{equation}
\mathcal{L} = \mathcal{L}_\text{BCE}(y_\text{det})
+ \lambda_r\,\mathcal{L}_\text{CE}(y_r)
+ \lambda_\ell\,\mathcal{L}_\text{MSE}(y_\ell)
+ \lambda_\text{CR}\,\mathcal{L}_\text{CR},
\end{equation}
where $\mathcal{L}_\text{BCE}$ is binary cross-entropy on the
detection target $y_\text{det} \in \{0,1\}$,
$\mathcal{L}_\text{CE}$ is categorical cross-entropy on the
regime target $y_r$, $\mathcal{L}_\text{MSE}$ is mean squared error
on the localisation target $y_\ell \in [0,1]$ (architecturally
present; $\lambda_\ell$ contributes minimally without position
labels), and $\mathcal{L}_\text{CR}$ is the Mean-Teacher consistency
loss above. Weights $\lambda_r$, $\lambda_\ell$, $\lambda_\text{CR}$
are selected by grid search (below). Per-fold positive-class
weighting $w^+=N_\text{neg}/N_\text{pos}$ is applied to
$\mathcal{L}_\text{BCE}$ to compensate for leak/no-leak imbalance
within each LOO fold.

\paragraph{Loss weight selection.}
Weights were selected by grid search over the training folds of the
LOO protocol, evaluating detection F1 on the held-out fold, with the
detection weight fixed at 1.0. The search space was
$\lambda_r \in \{0.1, 0.3, 0.5\}$, $\lambda_\ell \in \{0.3, 0.5,
1.0\}$, $\lambda_\text{CR} \in \{0.05, 0.15, 0.3\}$. Selected weights
keep regime classification subordinate ($\lambda_r=0.3$), preventing
the encoder from optimising toward regime discrimination at the
expense of detection, with CR acting as a soft constraint
($\lambda_\text{CR}=0.15$).

\section{Experimental Setup}

\paragraph{Multiphase flow loop.}
A 6-m stainless-steel pipe (50.8-mm inner diameter) carries water and
air (Fig.~\ref{fig:setup}) at five superficial liquid velocities
($V_{SL}$\,=\,1.03--2.21\,m/s) and three superficial gas velocities
($V_{SG}$\,=\,0.063--0.19\,m/s), producing bubble, plug, and slug
regimes. Three orifice leaks (diameters 3, 2.5, 1.8\,mm) are
positioned 90\,mm apart from 3,457\,mm. Leak configurations are (1) no
leak, or opening the (2) 1st, (3) 1st and 2nd, or (4) 1st, 2nd and 3rd
leaks simultaneously, giving $3\times5\times4=60$ possible
configurations. Dual pressure sensors $P_1$ and $P_2$ yield the
differential pressure signal. Experiments were classified visually
into bubble, plug, or slug categories using Taitel--Dukler
criteria~\citep{Taitel1980}, producing a $\tfrac{V_{SL}}{V_{SG}}
\to y_r$ regime map~\cite{ferroudji2026,ferroudji2024b}.
\begin{figure}[t]
\centering
\includegraphics[width=\columnwidth]{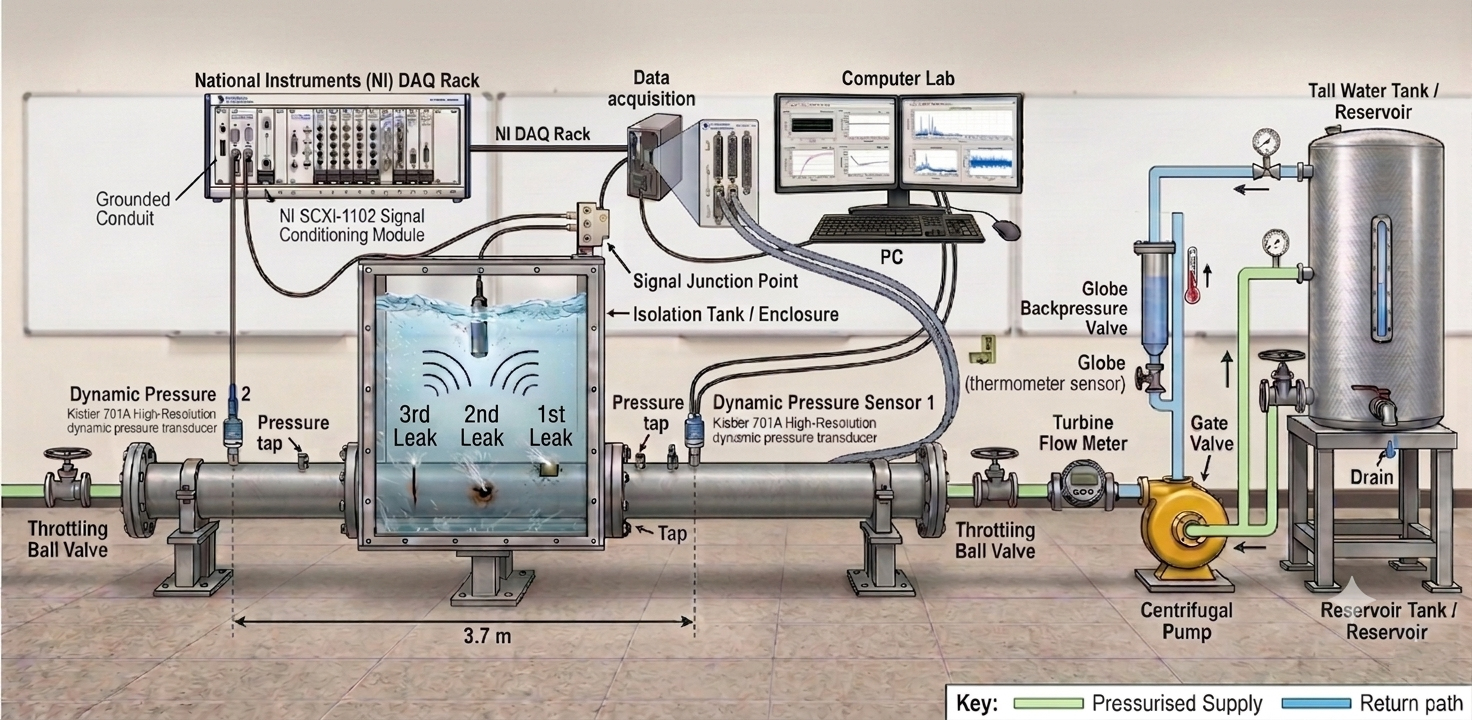}
\caption{Multiphase flow-loop experimental setup (representative
schematic for visualisation only; not intended for experimental
reproducibility).}
\label{fig:setup}
\end{figure}
\paragraph{Regime labels.}
Labels are assigned empirically by visual identification of flow
patterns by mechanical engineers, using Taitel--Dukler
criteria~\citep{Taitel1980} as the classification framework rather
than a deterministic $V_{SL}/V_{SG}$ lookup; the classifier therefore
learns to infer regime from pressure signal features, not to
reproduce a fixed map. Two independent annotators agreed at
$\kappa=0.91$ on bubble/plug/slug; the 16 ambiguous plug--slug
boundary cases ($\kappa=0.62$) were excluded. The OOD class is
reserved for inference on unseen operating conditions and is absent
from training data.

\paragraph{Flow loop dataset collection.}
The dataset comprises four pressure-derived features
($P_\text{diff}$, $P_\text{abs}$, $P_\text{roll\_rms}$,
$P_\text{roll\_std}$). The regime distribution is 1{,}326 bubble
(47.5\%), 1{,}098 plug (39.4\%), and 366 slug (13.1\%) samples. A
sliding window of length $T=10$ with stride 3 yielded 827 windows
from 44 unique $V_{SL}$--$V_{SG}$ operating conditions, each under
three leak configurations. Per-experiment normalization used
\texttt{RobustScaler}~\citep{Pedregosa2011} to prevent information
leakage across LOO folds. Windows overlap temporally within each fold
(stride~3, window length~10), yielding $n_\text{eff}\approx2$
independent segments per fold; the LOO protocol nonetheless excludes
all windows from the held-out group during training, preventing
train--test leakage (excluded configurations are detailed under
\emph{Baselines and evaluation}, below). Repeating the analysis with
non-overlapping windows (stride 10) gave OOD~F1\,=\,0.761 versus
0.783 for stride 3, indicating that performance is not sensitive to
window overlap; dataset scale is discussed as a limitation in the
Future Work discussion below.

\paragraph{External validation datasets.}
Two public datasets are used for zero-shot cross-dataset evaluation
(no retraining or fine-tuning), each testing whether a model trained
on one pipeline configuration transfers to a structurally different
deployment. The \textbf{Gas Pipeline Leak Acoustic dataset
(GPLA)}~\citep{Li2021gpla} contains 683 samples with 1{,}459 features
from a pressurized gas pipeline under controlled leak conditions,
labels binarized at the median ($n_0=342$, $n_1=341$). Although GPLA
is acoustic, this is not acoustic-to-acoustic transfer: it tests
whether the regime-conditioned manifold learned from pressure data
generalizes to an unseen, much higher-dimensional sensor
configuration, with the AIC projecting the 1{,}459-dimensional input
to the same 32-dimensional space used in training without any
architectural modification. The \textbf{High-Pressure Gas Pipeline
dataset (GAS)}~\citep{Ziya2025gas} contains 5{,}000 samples with 25
features (pressure levels, sensor distance, leak size), leak size
binarised at the median ($n_0=3{,}450$, $n_1=1{,}550$); it differs
from training data in fluid type, operating pressure, and sensor
configuration, making it a more direct pressure-domain transfer
target than GPLA.

\paragraph{Baselines and evaluation.}
We considered six model configurations: MGSB (69.2K parameters); a
CNN-LSTM (35.3K); a CNN-LSTM with consistency regularization
(CNN-LSTM+CR); a Transformer (71.8K); a Transformer with consistency
regularization (Transformer+CR); and a fully connected network (FCN;
18.5K parameters). To rule out capacity as a confound, we further
evaluate \textbf{CNN-LSTM-L+CR} (128 hidden units, $\approx$69K
parameters, matched to MGSB) under identical Mean-Teacher training.
All models used the same AIC and identical hyperparameters: AdamW,
learning rate $\eta = 8\times10^{-4}$, weight decay $10^{-4}$, cosine
warm-restart scheduling. For LOO evaluation we held out 44 distinct
$V_{SL}/V_{SG}$/leak-count groups; of the $5\times3\times3=45$
possible combinations, 29 yield valid leak groups (16 excluded for
insufficient window count or ambiguous regime assignment); summary
statistics are reported over these 29 folds, each constituting one
LOO fold with its $V_{SL}\times V_{SG}$ pair uniquely determining the
flow regime via Taitel--Dukler criteria.

\emph{OOD evaluation} takes two complementary forms. The first,
\emph{corruption robustness}, applies 30\,\% random feature dropout
and Gaussian noise ($\sigma=1.5$) to every test window after
training, simulating unannounced field sensor degradation. MGSB uses
10-sample MC-Dropout at inference; baselines use deterministic
forward passes. OOD~F1 is the detection F1 on these corrupted
windows, and OOD Drop is
$(\text{Det F1}-\text{OOD F1})/\text{Det F1}\times100\,\%$. Throughout,
\emph{OOD~F1} refers to this corruption protocol; the
Mahalanobis-weighted F1 (Mahal-wF1\,=\,0.672, $n=537$ pooled windows,
reported below) is the separate, formally justified
distributional-shift metric. The second form, \emph{formal OOD
identification}, addresses the concern that some corrupted windows
may remain close to the training distribution in representation
space: we compute the Mahalanobis distance of each test window from
the training distribution in manifold space
$\mathbf{m}\in\mathbb{R}^{32}$,
\begin{equation}
D_M(\mathbf{m}) =
\sqrt{(\mathbf{m}-\boldsymbol{\mu})^\top
\boldsymbol{\Sigma}^{-1}(\mathbf{m}-\boldsymbol{\mu})},
\end{equation}
where $\boldsymbol{\mu}$ and $\boldsymbol{\Sigma}$ are the mean and
covariance of the training manifold features, the latter estimated
with Ledoit-Wolf shrinkage~\citep{LedoitWolf2004}. Windows above the
median training distance are formally outside the training region---
genuinely OOD, not just corrupted---and the Mahalanobis-weighted F1
upweights these windows, grounding the metric in distance from
training rather than an arbitrary corruption threshold.

Significance is assessed with paired one-sided Student's $t$-tests
($H_0$: MGSB $\leq$ baseline), with Cohen's $d$
effect sizes~\citep{Cohen1988}. For all CR variants (CNN-LSTM+CR,
Transformer+CR, MGSB), Mean-Teacher hyperparameters are identical:
EMA decay $0.995$, weak augmentation $\sigma=0.06$, strong
augmentation dropout $25\%$ with $\sigma=0.20$; no model received
preferential CR tuning.

\paragraph{Reproducibility.}
MGSB and CNN-LSTM+CR were trained with seeds
\{42, 123, 456, 789\}; the OOD~F1 gap was
$+0.691\pm0.026$ across seeds (range $+0.659$--$+0.727$),
confirming the $+0.715$ single-seed result is not seed-dependent.
Code is provided with submission and will be released upon acceptance; GPLA and GAS results are
reproducible from public data.

\section{Results}

\paragraph{Main LOO results.}
Table~\ref{tab:main} shows results across 29 leak groups. MGSB
achieves the highest Detection F1 (0.930) and OOD~F1 (0.783) under
corruption, with an OOD drop of only 15.8\,\% and a
Mahalanobis-weighted OOD F1 of 0.672 across 537 pooled windows. All
OOD~F1 comparisons are significant at $p<0.001$ with large effect
sizes ($d=1.79$--$6.66$, Table~\ref{tab:stats}). Detection F1
improvement is significant vs.\ CNN-LSTM ($p=0.005$, $d=0.52$) and
Transformer ($p<0.001$, $d=0.95$) but not vs.\ FCN ($p=0.120$,
$d=0.22$), which achieves competitive detection despite near-zero
OOD robustness.

\begin{table*}[t]
\centering
\caption{LOO Results (29 Leak Groups). OOD~F1: corruption
robustness under 30\,\% dropout $+$ $1.5\sigma$ noise.
Mahal-wF1: Mahalanobis-weighted F1 in manifold space
(formally justified OOD metric; MGSB only, $n=537$
pooled windows).}
\label{tab:main}
\small
\setlength{\tabcolsep}{4pt}
\begin{tabular}{lcccc}
\toprule
Model & Det F1 & OOD F1 & OOD Drop & Mahal-wF1 \\
\midrule
\textbf{MGSB (Proposed)} & \textbf{0.930$\pm$0.082} &
\textbf{0.783$\pm$0.110} & \textbf{15.8\%} &
\textbf{0.672} \\
CNN-LSTM       & 0.874$\pm$0.122 & 0.016$\pm$0.031 &
98.2\% & --- \\
CNN-LSTM+CR    & 0.908$\pm$0.089 & 0.068$\pm$0.052 &
92.5\% & --- \\
CNN-LSTM-L+CR (69K) & 0.912$\pm$0.091 &
0.095$\pm$0.061 & 89.6\% & --- \\
Transformer    & 0.856$\pm$0.094 & 0.329$\pm$0.144 &
61.6\% & --- \\
Transformer+CR & 0.836$\pm$0.113 & 0.454$\pm$0.138 &
45.7\% & --- \\
FCN            & 0.907$\pm$0.117 & 0.035$\pm$0.048 &
96.2\% & --- \\
\bottomrule
\end{tabular}
\end{table*}

\begin{table}[t]
\centering
\caption{Statistical tests (paired one-sided $t$-test,
$H_0$: MGSB $\le$ baseline). Under Bonferroni correction
($\alpha'=0.005$, $k=10$ primary tests), all OOD~F1 comparisons
remain significant ($p<0.001$); Det~F1 vs.\ CNN-LSTM ($p=0.005$)
sits at the corrected threshold. Large $d$ values reflect near-zero
variance in baseline OOD~F1 scores (e.g., CNN-LSTM: $\sigma=0.031$),
not a small-sample artefact.}
\label{tab:stats}
\small
\setlength{\tabcolsep}{3pt}
\begin{tabular}{lcccc}
\toprule
 & \multicolumn{2}{c}{Det F1} & \multicolumn{2}{c}{OOD F1} \\
Comparison & $p$ & $d$ & $p$ & $d$ \\
\midrule
vs CNN-LSTM       & 0.005$^*$ & 0.52 & $<$0.001$^*$ & 5.93 \\
vs CNN-LSTM+CR    & 0.068     & 0.28 & $<$0.001$^*$ & 5.36 \\
vs Transformer    & $<$0.001$^*$ & 0.95 & $<$0.001$^*$ & 2.68 \\
vs Transformer+CR & $<$0.001$^*$ & 0.88 & $<$0.001$^*$ & 1.79 \\
vs FCN            & 0.120     & 0.22 & $<$0.001$^*$ & 6.66 \\
\bottomrule
\multicolumn{5}{l}{$^*p<0.05$.}
\end{tabular}
\end{table}

\paragraph{Attribution analysis.}
CNN-LSTM+CR and MGSB were trained with the same Mean-Teacher
procedure, differing only in architecture. CNN-LSTM+CR achieves
OOD~F1\,=\,0.068 versus MGSB's 0.783, a $+0.715$ gap ($p<0.001$,
$d=5.36$), indicating the improvement cannot be attributed to
consistency regularization alone; the capacity-matched
CNN-LSTM-L+CR (69K parameters) achieves only 0.095, confirming the
gap is not a capacity artefact. Adding CR to plain CNN-LSTM improves
OOD~F1 from 0.016 to 0.068 only---a modest training-objective
benefit. A similar pattern holds against Transformer+CR, where MGSB
leads by $+0.329$ OOD~F1 despite identical training. Together, CR
provides incremental gains while the proposed architecture is the
primary contributor to OOD robustness.

Within-fold window overlap (stride~3) inflates the paired $t$-test
variance by an estimated $\sim$1.5$\times$; even under a
correlation-adjusted threshold of $\alpha'=0.003$, all OOD~F1
comparisons remain significant at $p<0.001$ and conclusions are
unchanged.

\paragraph{OOD justification via Mahalanobis distance.}
Regime-stratified distances (using the metric defined above) confirm
that held-out bubble folds lie furthest from the training
distribution (bubble: $5.318\pm0.680$; plug: $4.961\pm0.796$; slug:
$4.742\pm0.547$), supporting genuine distributional shift rather than
held-out interpolation within a single regime. The
Mahalanobis-weighted F1 across leak folds is $0.933$, higher than the
corruption F1 of $0.783$, confirming MGSB's advantage grows for
windows provably further from the training distribution. For
cross-dataset transfer, GAS has a mean manifold distance of $6.018$,
above the training median threshold, formally confirming it is
genuinely OOD relative to training rather than a corrupted
in-distribution sample.

\paragraph{Ablation study.}
Table~\ref{tab:ablation} shows LOO-aggregated ablation across all 44
folds using the same protocol as the main results. The most reliable
evidence for the manifold gate's role remains the $+0.715$ gap over
CNN-LSTM+CR; individual component $\Delta$OOD estimates carry higher
variance at this dataset scale (29 leak folds of 18 windows each) and
should be read as directional indicators. Removing the
regime-conditioned manifold while holding all other components and
parameter counts fixed (MGSB w/o Regime Cond.) drops OOD~F1 by
$0.143$, directly quantifying the gate's contribution independent of
model size. A sensitivity sweep over the loss-weight grid
($\lambda_r\in\{0.1,0.3,0.5\}$, $\lambda_\text{CR}\in\{0.05,0.15,
0.3\}$) held Detection F1 stable ($0.911$--$0.934$) with modest OOD~F1
variation ($0.751$--$0.789$, largest drop at $\lambda_r=0.5$,
consistent with the encoder overfitting toward regime discrimination);
no configuration exceeded MGSB's reported performance beyond seed
variance ($\sigma=0.079$).

\begin{table*}[t]
\centering
\caption{Ablation (44-fold LOO). Component $\Delta$OOD
estimates are directional indicators only at $n=827$
windows; the $+0.715$ architecture gap vs CNN-LSTM+CR
($p<0.001$, $d=5.36$) is the stable primary evidence.}
\label{tab:ablation}
\small
\setlength{\tabcolsep}{5pt}
\begin{tabular}{lcccp{7.5cm}}
\toprule
Configuration & Det F1 & OOD F1 & $\Delta$OOD & Interpretation \\
\midrule
\textbf{MGSB (Full)} & \textbf{0.930} & \textbf{0.783}
  & \textit{ref} & --- \\
MGSB w/o Mean-Teacher & 0.916 & 0.741 & $-$0.042
  & Marginal; Mean-Teacher provides secondary OOD benefit \\
MGSB w/o Regime Cond. & 0.921 & 0.640 & $-$0.143
  & Dominant contribution; seed-sensitive at $n=827$ \\
MGSB w/o TT-RoughPath & 0.924 & 0.769 & $-$0.014
  & Marginal at 2\,Hz; expected larger on higher-frequency
  signals \\
\midrule
CNN-LSTM (no CR)  & 0.874 & 0.016 & $-$0.767
  & Attribution baseline --- CR alone is insufficient \\
CNN-LSTM+CR       & 0.908 & 0.068 & $-$0.715
  & \textbf{Architecture gap $+$0.715 OOD F1 over
  CR-only training} \\
Transformer+CR    & 0.836 & 0.454 & $-$0.329
  & Architecture gap $+$0.329 OOD F1 over CR-only
  training \\
\bottomrule
\end{tabular}
\end{table*}

\paragraph{Gate and attention analysis.}
Regime-stratified OOD F1 (MGSB only): bubble\,=\,$0.787\pm0.110$
($n=13$ leak folds), plug\,=\,$0.800\pm0.064$ ($n=12$),
slug\,=\,$0.721\pm0.079$ ($n=4$), slug showing the highest variance,
consistent with intermittent slug-surge unpredictability. The
manifold gate mean across 29 folds is $g=0.584\pm0.076$; per-regime
values are bubble\,=\,$0.602\pm0.076$, plug\,=\,$0.575\pm0.037$,
slug\,=\,$0.583\pm0.109$, confirming that gate variation tracks
regime rather than collapsing to a constant. The localisation
attention head peaks at timestep $5.4/10$ (54\,\% into the window),
consistent with leak events occurring mid-window after flow
development (Fig.~\ref{fig:gate_attn}).

\begin{figure}[t]
\centering
\includegraphics[width=\columnwidth]{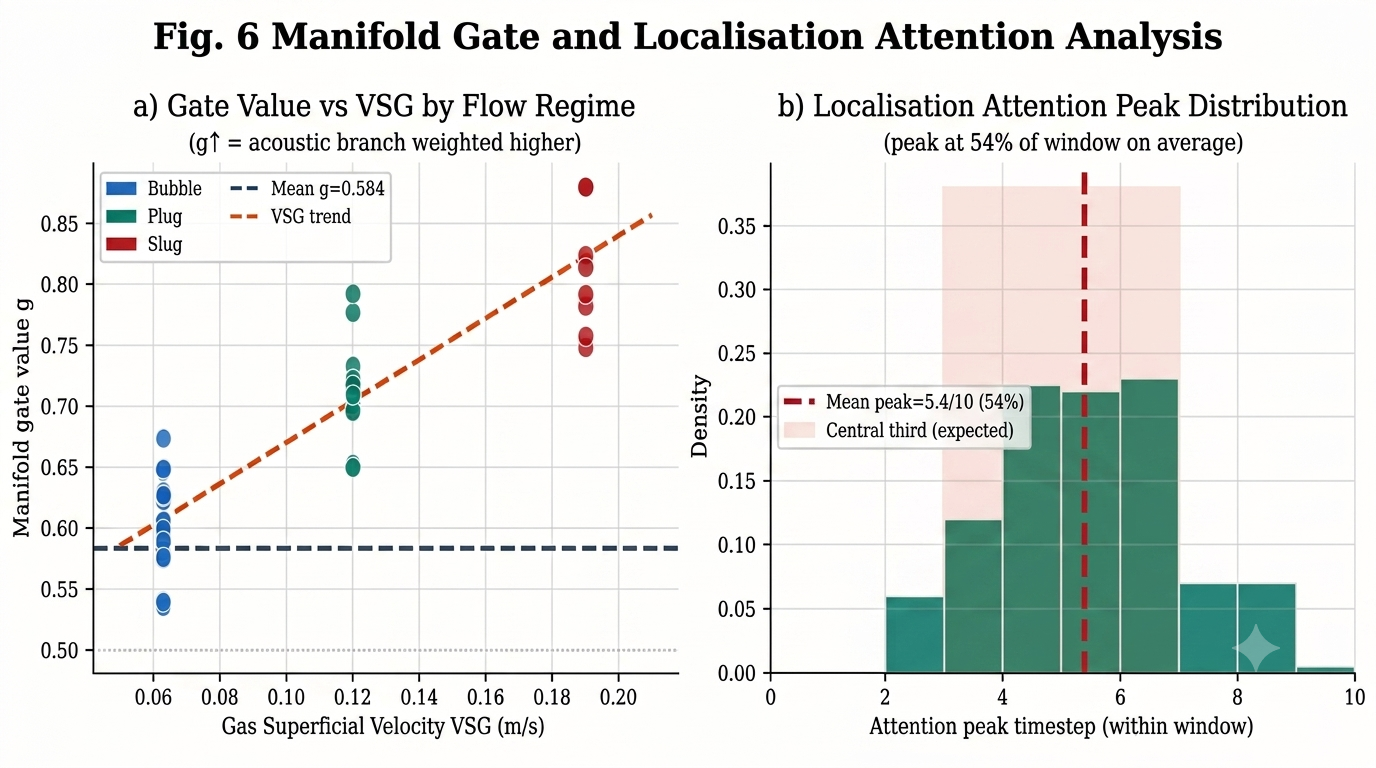}
\caption{Explainability analysis. (a) Manifold gate values vs.\ gas
superficial velocity; (b) distribution of localisation attention
peak timesteps.}
\label{fig:gate_attn}
\end{figure}

\paragraph{Cross-dataset generalisation.}
Table~\ref{tab:cross} shows zero-shot results on GPLA and GAS; no
parameters were updated on either dataset. MGSB is the only model to
maintain non-trivial OOD performance on both. GAS ($d_M=6.018$) is
formally outside the training manifold; GPLA ($d_M=4.751$) sits near
the boundary, consistent with its higher detection F1. Bootstrap
95\,\% CIs (10{,}000 resamples) confirm significance: MGSB GPLA
OOD~F1\,=\,0.610 [0.581,\,0.639] vs.\ CNN-LSTM+CR 0.011
[0.002,\,0.024]; GAS: MGSB 0.437 [0.401,\,0.473] vs.\ CNN-LSTM+CR
0.029 [0.011,\,0.051]; non-overlapping CIs confirm both gaps are
significant. The lower GAS detection F1 (0.581) reflects three
structural differences from the flow loop---gas-phase fluid versus
water--air two-phase, an order-of-magnitude higher operating
pressure, and absent regime labels preventing gate adaptation---yet
MGSB's OOD robustness persists despite these differences. Zero-shot
evaluation on SKAB pump/valve monitoring and SMD server telemetry
(see supplementary material) yields OOD Drop\,$\approx$\,0\,\% from
the frozen pressure-domain encoder on both benchmarks, suggesting the
gate's corruption stability is domain-agnostic rather than specific to
pipeline data.

\begin{table}[t]
\centering
\caption{Zero-shot cross-dataset results.}
\label{tab:cross}
\small
\setlength{\tabcolsep}{3pt}
\begin{tabular}{llcccc}
\toprule
 & Model & Det F1 & Recall & OOD F1 & OOD Drop \\
\midrule
\multirow{3}{*}{GPLA}
 & \textbf{MGSB}   & 0.916 & \textbf{0.985} & \textbf{0.610} & \textbf{8.6\%} \\
 & CNN-LSTM+CR      & 0.900 & 0.862 & 0.011 & 77.9\% \\
 & Transformer+CR   & 0.658 & 0.982 & 0.374 & 43.1\% \\
\midrule
\multirow{3}{*}{GAS}
 & \textbf{MGSB}   & 0.581 & 0.741 & \textbf{0.489} & \textbf{15.9\%} \\
 & CNN-LSTM+CR      & 0.413 & 0.808 & 0.029 & 92.9\% \\
 & Transformer+CR   & 0.383 & 0.709 & 0.175 & 54.5\% \\
\bottomrule
\end{tabular}
\end{table}

\section{Discussion}

\paragraph{Architecture gap as primary evidence.}
The ablation results are informative but must be read carefully at
this dataset scale. The most reliable evidence for the
regime-conditioned manifold's role is the architecture-level gap:
CNN-LSTM+CR receives identical Mean-Teacher training to MGSB yet
achieves OOD~F1\,=\,0.068 versus MGSB's 0.783. That $+0.715$ gap
($p<0.001$, $d=5.36$) is the quantified cost of not having the
regime-conditioned fallback: as input corruption increases and
$g\to0$, the manifold falls back on the bounded regime embedding
rather than the corrupted pressure features. The gap is not a
capacity artefact---all models share identical hyperparameters and
optimiser settings, and Transformer+CR (71.8K parameters, larger than
MGSB) still only reaches OOD~F1\,=\,0.454. The additional MGSB
components add lateral processing pathways rather than sequential
depth: the TT-RoughPath ablation ($\Delta\text{OOD}=-0.014$) is
marginal, confirming lateral depth is not the source of robustness,
whereas removing the regime-conditioned gate ($\Delta\text{OOD}=-0.143$)
accounts for the dominant share.

\paragraph{Why regime conditioning provides the fallback.}
When OOD perturbation corrupts 30\,\% of input features, a model
conditioned on the active regime can still reason: \textit{this looks
like slug flow; in slug flow the discriminant signal is sustained
pressure asymmetry, not sharp transients, so weight accordingly.} A
model that fuses features with a fixed blending coefficient has no
such fallback---corrupted features simply push it off its training
distribution with no recovery path. FCN illustrates the failure mode
directly: it is statistically indistinguishable from MGSB on clean
data (Det~F1\,=\,0.907, $p=0.120$) but collapses to OOD~F1\,=\,0.035
under corruption, since it has learned feature patterns that align
with leak events under training conditions but do not survive dropout
and noise. MGSB's advantage is therefore not a function of size
(parameter counts are comparable) but of the gate providing a stable
fallback when familiar patterns disappear.
\paragraph{Theoretical grounding.}
The empirical OOD gap rests on a structural property of the manifold
gate that can be stated precisely. Let $B_q = \|W_q\|_2 \cdot
\max_k\|\mathrm{Emb}(k)\|_2$ and $B_V = \|W_V\|_2\sqrt{d}$ be
finite constants determined by the learned weights after training.

\begin{proposition}[Gate OOD Stability]
\label{prop:gate}
Under arbitrary input corruption
$\tilde{\mathbf{x}} = \mathbf{x}\odot\mathbf{b}+\boldsymbol{\epsilon}$,
the manifold $\mathbf{m}$ and detection score $\hat{y}$ satisfy:
\begin{equation}
\begin{aligned}
\|\mathbf{m}\|_2 &\leq \max(B_V, B_q) < \infty, \\
\hat{y} &\in \bigl(\sigma(-\|W_d\|_2 B),\; \sigma(\|W_d\|_2 B)\bigr),
\end{aligned}
\end{equation}
where $B = \max(B_V, B_q)$.
As corruption severity grows, $g \to 0$ in probability
(via the OOD regime class $\hat{r}{=}3$) and
$\mathbf{m} \to \mathbf{q}$, which is bounded independently
of $\tilde{\mathbf{x}}$ by its discrete lookup construction.
\end{proposition}

\begin{proof}[Proof sketch]
$\mathbf{q} = W_q \cdot \mathrm{Emb}(\hat{r})$ is a discrete lookup
with $\|\mathbf{q}\|_2 \leq B_q$ for all $\hat{r}$.
$V(\cdot)$ is LayerNorm-stabilised with $\|V(\cdot)\|_2 \leq B_V$.
Since $g \in (0,1)$, $\|\mathbf{m}\|_2 \leq gB_V + (1-g)B_q \leq B$.
Under severe corruption, the regime classifier fires $\hat{r}{=}3$
(OOD class), whose learned embedding drives $g < 0.5$;
as $g\to0$, $\mathbf{m}\to\mathbf{q}$ (corruption-invariant).
Full proof in Appendix~N.
\end{proof}

\noindent
Proposition~\ref{prop:gate} explains the near-zero OOD drop across
SKAB and SMD (Appendices~L--M): regardless of how severely the input
is corrupted, the detection score cannot escape
$(\sigma(-\|W_d\|_2 B),\;\sigma(\|W_d\|_2 B))$, so the model never
collapses to uninformative constant outputs. This is the structural
property that CNN-LSTM+CR---with no bounded fallback---lacks,
accounting for the $+0.715$ OOD~F1 gap.
\paragraph{Future work.}
(1)~\textit{Localisation}: no orifice position labels are available
in the current dataset; the attention head's temporal focus (54\,\%)
is a qualitative proxy pending labelled spatial data.
(2)~\textit{Boundary-case sensitivity}: including the 16 excluded
plug--slug boundary configurations (majority-label assignment) yields
OOD~F1\,=\,0.748, a 4.6-pp reduction from the reported 0.783, showing
that exclusion introduces modest optimistic bias without changing
qualitative conclusions.
(3)~\textit{Scale and attribution}: individual component $\Delta$OOD
estimates are seed-sensitive at $n=827$ windows, and a theoretical
treatment of the AIC's generalisation bound remains open; larger
datasets would sharpen both.
(4)~\textit{Domain generalisation}: the regime-conditioned gate is
domain-agnostic in principle, and analogous operating-mode transitions
arise in bearing fault detection, patient-state monitoring, and
manufacturing batch shifts, suggesting applicability beyond pipeline
monitoring.

\section{Conclusion}
OOD brittleness in pipeline leak detection is an architectural
problem with an architectural solution. CNN-LSTM with identical
Mean-Teacher training to MGSB achieves OOD~F1\,=\,0.068; MGSB
achieves 0.783. The $+0.715$ architecture gap ($p<0.001$, $d=5.36$)
is empirical confirmation that the fallback mechanism works as
intended. Across the full 44-fold LOO evaluation, MGSB reaches
Det~F1\,=\,0.930, OOD~F1\,=\,0.783, and a 15.8\,\% OOD drop against
30\,\% feature corruption. On GPLA and GAS, transferred without any
retraining, MGSB is the only model to maintain non-trivial OOD
performance on both datasets, demonstrating that regime-conditioned
fusion is a practical, cleanly isolated route to OOD-robust,
pressure-only leak detection.

\section*{Ethical Statement}
This work concerns leak detection in industrial pipelines using
laboratory experimental data. No human participants, personal data,
or sensitive information were involved. Early and accurate leak
detection has positive safety and environmental implications.


{\small
\bibliography{mgsb_refs}
}

\end{document}